\documentclass{article}

\usepackage{microtype}
\usepackage{graphicx}
\usepackage{subfigure}
\usepackage{booktabs} 
\usepackage{enumitem}

\usepackage{hyperref}
\usepackage{stmaryrd}
\usepackage{amsmath}
\usepackage{amsthm}
\usepackage{amssymb}
\usepackage{cleveref}
\usepackage{mathrsfs} 
\newtheorem{assumption}{\textbf{H}\hspace{-3pt}}
\Crefname{assumption}{\textbf{H}\hspace{-3pt}}{\textbf{H}\hspace{-3pt}}
\crefname{assumption}{\textbf{H}}{\textbf{H}}

\usepackage{todonotes}

\newtheorem{proposition}{Proposition}
\newtheorem{remark}{Remark}
\DeclareMathOperator*{\argmax}{arg\,max}

\usepackage[accepted]{icml2019}

\icmltitlerunning{Generalized Sliced Wasserstein Distances}

\def\msw{\text{max-$SW_p$}}
\def\mgsw{\text{max-$GSW_p$}}

\begin{document}

\twocolumn[
\icmltitle{Generalized Sliced Wasserstein Distances}





\begin{icmlauthorlist}
\icmlauthor{Soheil Kolouri}{hrl}
\icmlauthor{Kimia Nadjahi}{paris}
\icmlauthor{Umut \c{S}im\c{s}ekli}{paris}
\icmlauthor{Roland Badeau}{paris}
\icmlauthor{Gustavo K. Rohde}{uva}
\end{icmlauthorlist}

\icmlaffiliation{hrl}{HRL Laboratories, LLC., Malibu, CA, USA}
\icmlaffiliation{paris}{T\'{e}l\'{e}com ParisTech, Paris, France}
\icmlaffiliation{uva}{University of Virginia
Charlottesville, VA, USA}

\icmlcorrespondingauthor{Soheil Kolouri}{skolouri@hrl.com}

\icmlkeywords{Wasserstein distances, Slicing theorem}

\vskip 0.3in
]



\printAffiliationsAndNotice{}  

\begin{abstract}
The Wasserstein distance and its variations, e.g., the sliced-Wasserstein (SW) distance, have recently drawn attention from the machine learning community. The SW distance, specifically, was shown to have similar properties to the Wasserstein distance, while being much simpler to compute, and is therefore used in various applications including generative modeling and general supervised/unsupervised learning. In this paper, we first clarify the mathematical connection between the SW distance and the Radon transform. We then utilize the generalized Radon transform to define a new family of distances for probability measures, which we call generalized sliced-Wasserstein (GSW) distances. We also show that, similar to the SW distance, the GSW distance can be extended to a maximum GSW (max-GSW) distance. We then provide the conditions under which GSW and max-GSW distances are indeed distances. Finally, we compare the numerical performance of the proposed distances on several generative modeling tasks, including SW flows and SW auto-encoders. 
\end{abstract}

\section{Introduction}
\label{sec:intro}

The Wasserstein distance has its roots in optimal transport (OT) theory \cite{villani2008optimal} and forms a metric between two probability measures. It has attracted abundant attention in data sciences and machine learning due to its convenient theoretical properties and applications on many domains \cite{solomon2014wasserstein,frogner2015learning,montavon2016wasserstein,kolouri2017optimal, courty2017optimal, peyre2018computational,schmitz2018wasserstein}, especially in implicit generative modeling such as OT-based generative adversarial networks (GANs) and variational auto-encoders \cite{arjovsky2017wasserstein,bousquet2017optimal,gulrajani2017improved,tolstikhin2018wasserstein}. 


While OT brings new perspectives and principled ways to formalize problems, the OT-based methods usually suffer from high computational complexity. The Wasserstein distance is often the computational bottleneck and it turns out that evaluating it between multi-dimensional measures is numerically intractable in general. This important computational burden is a major limiting factor in the application of OT distances to large-scale data analysis. Recently, several numerical methods have been proposed to speed-up the evaluation of the Wasserstein distance. For instance, entropic regularization techniques \cite{cuturi2013sinkhorn,cuturi2015smootheddual,solomon2015convolutional} provide a fast approximation to the Wasserstein distance by regularizing the original OT problem with an entropy term. The linear OT approach, \cite{wang2013linear,kolouri2016continuous} further simplifies this computation for a given dataset by a linear approximation of pairwise distances with a functional defined on distances to a reference measure. Other notable contributions towards computational methods for OT include multi-scale and sparse approximation approaches \cite{oberman2015efficient,schmitzer2016sparse}, and Newton-based schemes for semi-discrete OT \cite{levy2015,kitagawa2016convergence}. 


There are some special favorable cases where solving the OT problem is easy and reasonably cheap. In particular, the Wasserstein distance for one-dimensional probability densities has a closed-form formula that can be efficiently approximated. This nice property motivates the use of the sliced-Wasserstein distance \cite{bonneel2015sliced}, an alternative OT distance obtained by computing infinitely many \emph{linear projections} of the high-dimensional distribution to one-dimensional distributions and then computing the average of the Wasserstein distance between these one-dimensional representations. While having similar theoretical properties \cite{bonnotte2013unidimensional}, the sliced-Wasserstein distance has significantly lower computational requirements than the classical Wasserstein distance. Therefore, it has recently attracted ample attention and successfully been applied to a variety of practical tasks \cite{bonneel2015sliced,kolouri2016sliced,carriere2017sliced,karras2017progressive,csimcsekli2018sliced,deshpande2018generative,Kolouri_2018_CVPR,kolouri2018sliced}. 


As we will detail in the next sections, the linear projection process used in the sliced-Wasserstein distance is closely related to the Radon transform, which is widely used in tomography \cite{radon1917uber, helgason2011radon}. In other words, the sliced-Wasserstein distance is calculated via linear slicing of the probability distributions. However, the linear nature of these projections does not guarantee an efficient evaluation of the sliced-Wasserstein distance: in very high-dimensional settings, the data often lives in a thin manifold and the number of randomly chosen linear projections required to capture the structure of the data distribution grows very quickly \cite{csimcsekli2018sliced}. Reducing the number of required projections would thus result in a significant performance improvement in sliced-Wasserstein computations.


{\bf Contributions.}
In this paper, we address the aforementioned computational issues of the sliced-Wasserstein distance and for the first time, we extend the linear slicing to \emph{non-linear} slicing of probability measures. Our main contributions are summarized as follows:
\begin{itemize}[leftmargin=*,topsep = 0pt, noitemsep]
\item Using the mathematics of the \emph{generalized} Radon transform \cite{beylkin1984inversion} we extend the definition of the sliced-Wasserstein distance to an entire class of distances, which we call the generalized sliced-Wasserstein (GSW) distance. We prove that replacing the linear projections with \emph{polynomial} projections will still yield a valid distance metric and we then identify general conditions under which the
GSW distance is a distance function.   \vspace{2pt}
\item We then show that, instead of using infinitely many projections as required by the GSW distance, we can still define a valid distance metric by using a \emph{single} projection, as long as the projection gives the maximal distance in the projected space. We aptly call this distance the max-GSW distance. The max-GSW distance vastly reduces the computational cost induced by the projection operations; however, it comes with an additional cost since it requires optimization over the space of projectors.   \vspace{2pt}
\item Due to their inherent non-linearity, the GSW and max-GSW distances are expected to capture the complex structure of high-dimensional distributions by using much less projections, which will reduce the overall computational burden in a significant amount. We verify this fact in our experiments, where we illustrate the superior performance of the proposed distances in both synthetic and real-data settings.  
\end{itemize}

\section{Background}

We review in this section the preliminary concepts and formulations needed to develop our framework, namely the $p$-Wasserstein distance, the Radon transform, the sliced $p$-Wasserstein distance and the maximum sliced $p$-Wasserstein distance. In what follows, we denote by $P_p(\Omega)$ the set of Borel probability measures with finite $p$'th moment defined on a given metric space $(\Omega,d)$ and by $\mu\in P_p(X)$ and $\nu\in P_p(Y)$ probability measures defined on $X,Y\subseteq \Omega$ with corresponding probability density functions $I_\mu$ and $I_\nu$, \textit{i.e.} $d\mu(x)=I_\mu(x)dx$ and $d\nu(y)=I_\nu(y)dy$. 

\subsection{Wasserstein Distance}

The $p$-Wasserstein distance, $p \in [1,\infty)$, between $\mu$ and $\nu$ is defined as the solution of the optimal mass transportation problem \cite{villani2008optimal}:
\begin{eqnarray}
    W_p(\mu,\nu)=\left(\operatorname*{inf}_{\gamma\in \Gamma(\mu,\nu)} \int_{X\times Y} d^p(x,y)d \gamma(x,y) \right)^{\frac{1}{p}}
\end{eqnarray}
where $d^p(\cdot, \cdot)$ is the cost function, and $\Gamma(\mu,\nu)$ is the set of all transportation plans $\gamma\in\Gamma(\mu,\nu)$ such that:
\begin{eqnarray*}
\begin{array}{lr}
    \gamma(A \times Y)= \mu(A) & \text{for any Borel subset } A\subseteq X \\
    \gamma(X \times B)= \nu(B) & \text{for any Borel subset } B\subseteq Y
\end{array}.
\end{eqnarray*}

Due to Brenier's theorem \cite{brenier1991polar}, for absolutely continuous probability measures $\mu$ and $\nu$ (with respect to the Lebesgue measure), the $p$-Wasserstein distance can be equivalently obtained from
\begin{eqnarray}
    W_p(\mu,\nu)=\left(\operatorname*{inf}_{f\in MP(\mu,\nu)} \int_{X} d^p\big(x,f(x)\big) d\mu(x)\right)^{\frac{1}{p}}
\end{eqnarray}
where $MP(\mu,\nu)=\{ f:X\rightarrow Y ~|~ f_\#\mu=\nu\}$ and $f_\#\mu$ represents the pushforward of measure $\mu$, characterized as
\begin{equation*}
    \int_{A}df_\#\mu(y) = \int_{f^{-1}(A)} d\mu(x) ~\text{for any Borel subset } A\subseteq Y.
\end{equation*}

Note that in most engineering and computer science applications, $\Omega$ is a compact subset of $\mathbb{R}^d$ and  $d(x,y)=|x-y|$ is the Euclidean distance. By abuse of notation, we will use $W_p(\mu, \nu)$ and $W_p(I_\mu,I_\nu)$ interchangeably. 

{\bf One-dimensional distributions:}\; The case of one-dimensional continuous probability measures is specifically interesting as the $p$-Wasserstein distance has a closed-form solution. More precisely, for one-dimensional probability measures, there exists a unique monotonically increasing transport map that pushes one measure to another. Let $F_\mu(x)=\mu((-\infty,x])=\int_{-\infty}^{x} I_\mu(\tau)d\tau$ be the cumulative distribution function (CDF) for $I_\mu$ and define $F_\nu$ to be the CDF of $I_\nu$. The optimal transport map is then uniquely defined as $f(x)= F_\nu^{-1}(F_\mu(x))$ and, consequently, the $p$-Wasserstein distance has an analytical form given as follows: 
\begin{eqnarray}
    W_p(\mu,\nu)&=&\left( \int_{X} d^p\big(x, F_\nu^{-1}(F_\mu(x))\big) d\mu(x) \right)^{\frac{1}{p}} \nonumber \\
    &=& \left(\int_{0}^1 d^p\big(F_\mu^{-1}(z), F_\nu^{-1}(z)\big)dz \right)^{\frac{1}{p}} \label{eq:wp1d}
\end{eqnarray}
where Eq.~\eqref{eq:wp1d} results from the change of variable $F_\mu(x)=z$. Note
that for empirical distributions, Eq. \eqref{eq:wp1d} is calculated by simply sorting the samples from the two distributions and calculating the average $d^p(\cdot,\cdot)$ between the sorted samples. This requires only $O(M)$ operations at best and $O(M\log M)$ at worst, where $M$ is the number of samples drawn from each distribution (see \citet{kolouri2018sliced} for more details). 
The closed-form solution of the $p$-Wasserstein distance for one-dimensional distributions is an attractive property that gives rise to the sliced-Wasserstein (SW) distance. Next, we review the Radon transform, which enables the definition of the SW distance. We also formulate an alternative OT distance called the maximum sliced-Wasserstein distance.

\subsection{Radon Transform}

The standard Radon transform, denoted by $\mathcal{R}$, maps a function  $I \in L^1(\mathbb{R}^d)$, where $$L^1(\mathbb{R}^d) = \{ I:\mathbb{R}^d \rightarrow \mathbb{R}\ / \int_{\mathbb{R}^d} |I(x)|dx < \infty\},$$ 
to the infinite set of its integrals over the hyperplanes of $\mathbb{R}^d$ and is defined as
\begin{eqnarray}
    \mathcal{R} I(t,\theta) = \int_{\mathbb{R}^d} I(x)\delta(t-\langle x, \theta \rangle)dx,
\label{eq:radon}
\end{eqnarray}
for $(t, \theta) \in \mathbb{R} \times \mathbb{S}^{d-1}$, where $\mathbb{S}^{d-1} \subset \mathbb{R}^d$ stands for the $d$-dimensional unit sphere, $\delta(\cdot)$ the one-dimensional Dirac delta function, and $\langle \cdot, \cdot \rangle$ the Euclidean inner-product. Note that $\mathcal{R}: L^1(\mathbb{R}^d)\rightarrow L^1(\mathbb{R}\times \mathbb{S}^{d-1})$. Each hyperplane can be written as:
\begin{equation}
    H(t, \theta) = \{x\in \mathbb{R}^d \ |\ \langle x, \theta \rangle = t\},
    \label{eq:hyperplanes}
\end{equation}
which alternatively can be interpreted as a level set of the function $g \in \mathbb{R}^d\times\mathbb{S}^{d-1}\rightarrow \mathbb{R}$ defined as $g(x, \theta) = \langle x, \theta \rangle$. For a fixed $\theta$, the integrals over all hyperplanes orthogonal to $\theta$ define a continuous function $\mathcal{R}I(\cdot,\theta) : \mathbb{R} \rightarrow \mathbb{R}$ which is a projection (or a slice) of $I$. 

The Radon transform is a linear bijection \cite{natterer1986mathematics,helgason2011radon} and its inverse $\mathcal{R}^{-1}$ is defined as:
\begin{eqnarray}
    I(x) &=& \mathcal{R}^{-1}\big(\mathcal{R}I(t,\theta)\big) \nonumber \\
    &=& \int_{\mathbb{S}^{d-1}} (\mathcal{R}I(\langle x, \theta \rangle,\theta)*\eta(\langle x, \theta \rangle) d\theta
\end{eqnarray}
where $\eta(\cdot)$ is a one-dimensional high-pass filter with corresponding Fourier transform $\mathcal{F}\eta(\omega) =  c|\omega|^{d-1}$, which appears due to the Fourier slice theorem \cite{helgason2011radon}, 
and `$*$' is the convolution operator. The above definition of the inverse Radon transform is also known as the filtered back-projection method, which is extensively used in image reconstruction in the biomedical imaging community. Intuitively each one-dimensional projection (or slice) $\mathcal{R}I(\cdot,\theta)$ is first filtered via a high-pass filter and then smeared back into $\mathbb{R}^{d}$ along $H(\cdot,\theta)$ to approximate $I$. The summation of all smeared approximations then reconstructs $I$. Note that in practice, acquiring an infinite number of projections is not feasible, therefore the integration in the filtered back-projection formulation is replaced with a finite summation over projections (\textit{i.e.}, a Monte-Carlo approximation).

\comment{\textcolor{red}{[Gustavo] As with the other paper, I think this section does not clarify anything mathematical, and could probably be moved towards the end, close to a computational section.}
\textcolor{blue}{[Kimia] I agree with Gustavo's comment: this paragraph is useful but at this stage of the paper, it might confuse the reader.}}

\comment{{\bf Radon transform of empirical PDFs:}\; The Radon transform of $I_\mu$ simply follows Equation \eqref{eq:radon}, where $\mathcal{R}I_\mu(\cdot,\theta)$ is a one-dimensional marginal distribution of $I_\mu$. However, in most machine learning applications we do not have access to the distribution $I_\mu$ but to a set of samples drawn from $I_\mu$ and denoted by $\{ x_n \}$. In such scenarios, kernel density estimation can be used to approximate $I_\mu$ from its samples:
\begin{eqnarray*}
I_\mu(x)\approx \frac{1}{N}\sum_{n=1}^N \phi(x-x_n)
\end{eqnarray*} 
where $\phi:\mathbb{R}^d\rightarrow \mathbb{R}^+$ is a density kernel such that $\int_{\mathbb{R}^d} \phi(x)dx=1$ (e.g., a Gaussian kernel). The Radon transform of  $I_\mu$ can then be approximated by:
\begin{eqnarray*}
    \mathcal{R}I_\mu(t,\theta)\approx \frac{1}{N}\sum_{n=1}^N \mathcal{R}\phi(t - \langle x_n, \theta \rangle, \theta).
\end{eqnarray*} 

Note that certain density kernels have an analytical Radon transform. For instance, for $\phi(x) = \delta(x)$, the Radon transform is $\mathcal{R}\phi(t,\theta)=\delta(t)$. Similarly, for Gaussian kernels and $\phi(x) = \mathcal{N}_d(0_d,\sigma^2I_{d\times d})$, the Radon transform is $\mathcal{R}\phi(t,\theta)=\mathcal{N}_1(0,\sigma^2)$. Moreover, given the high-dimensional nature of the problem, estimating the density $I$ in $\mathbb{R}^d$ requires a large number of samples. However,  the projections of $I$, $\mathcal{R}I(\cdot,\theta)$, are one-dimensional, therefore it may not be critical to have a large number of samples to estimate these one-dimensional densities. }

\subsection{Sliced-Wasserstein and Maximum Sliced-Wasserstein Distances} \label{subsection:SWandmaxSW}

The idea behind the sliced $p$-Wasserstein distance is to first, obtain a family of one-dimensional representations for a higher-dimensional probability distribution through linear projections (via the Radon transform), and then, calculate the distance between two input distributions as a functional on the $p$-Wasserstein distance of their one-dimensional representations (\textit{i.e.}, the one-dimensional marginal distributions). The sliced $p$-Wasserstein distance between $I_\mu$ and $I_\nu$ is then formally defined as:
 \begin{equation}
    SW_p(I_\mu,I_\nu)=\left( \int_{\mathbb{S}^{d-1}} W^p_p\big( \mathcal{R} I_\mu(.,\theta), \mathcal{R} I_\nu(.,\theta) \big) d\theta \right)^{\frac{1}{p}}
 \label{eq:SW}
 \end{equation}
 
This is indeed a distance function as it satisfies positive-definiteness, symmetry and the triangle inequality \cite{bonnotte2013unidimensional,kolouri2016sliced}. 

The computation of the SW distance requires an integration over the unit sphere in $\mathbb{R}^{d}$. In practice, this integration is approximated by using a simple Monte Carlo scheme that draws samples $\{ \theta_l \}$ from the uniform distribution on $\mathbb{S}^{d-1}$ and replaces the integral with a finite-sample average: 
\begin{equation}
    SW_p(I_\mu,I_\nu) \approx  \left(\frac{1}{L}\sum_{l=1}^L W^p_p\big(\mathcal{R}I_\mu(\cdot, \theta_l), \mathcal{R}I_\nu(\cdot, \theta_l)\big)\right)^{\frac{1}{p}}
\label{eq:empiricalSWD}
\end{equation}

The sliced $p$-Wasserstein distance has important practical implications: provided that $\mathcal{R}I_\mu(\cdot, \theta_l)$ and $\mathcal{R}I_\nu(\cdot, \theta_l)$ can be computed for any sample $\theta_l$, then the SW distance is obtained by solving several one-dimensional optimal transport problems, which have closed-form solutions. It is especially useful when one only has access to samples of a high-dimensional PDF $I$ and kernel density estimation is required to estimate $I$: one-dimensional kernel density estimation of PDF slices is a much simpler task compared to the direct estimation of $I$ from its samples. The downside is that as the dimensionality grows, one requires a larger number of projections to accurately estimate $I$ from $\mathcal{R}I(\cdot, \theta)$. In short, if a reasonably smooth two-dimensional distribution can be approximated using $L$ projections, then $\mathcal{O}(L^{d-1})$ projections are required to approximate a similarly smooth $d$-dimensional distribution for $d \geq 2$.

To further clarify this, let $I_\mu=\mathcal{N}(0, I_d)$ and $I_\nu=\mathcal{N}(x_0, I_d)$, $x_0 \in \mathbb{R}^d$, be two multivariate Gaussian densities with the identity matrix as the covariance matrix. Their projected representations are one-dimensional Gaussian distributions of the form $\mathcal{R}I_\mu(\cdot,\theta)=\mathcal{N}(0,1)$ and $\mathcal{R}I_\mu(\cdot,\theta)=\mathcal{N}(\langle \theta, x_0 \rangle, 1)$. It is therefore clear that $W_2(\mathcal{R}I_\mu(\cdot,\theta),\mathcal{R}I_\nu(\cdot,\theta))$ achieves its maximum value when $\theta=\frac{x_0}{\|x_0\|_2}$ and is zero for $\theta$'s that are orthogonal to $x_0$. On the other hand, we know that vectors that are randomly picked from the unit sphere are more likely to be nearly orthogonal in high-dimension. More rigorously, the following inequality holds: $Pr(| \langle \theta, \frac{x_0}{\|x_0\|_2} \rangle | < \epsilon) > 1-e^{(-d\epsilon^2)}$, which implies that for a high dimension $d$, the majority of sampled $\theta$'s would be nearly orthogonal to $x_0$ and therefore, $W_2(\mathcal{R}I_\mu(\cdot,\theta), \mathcal{R}I_\nu(\cdot,\theta))\approx 0$ with high probability. 

To remedy this issue, one can avoid uniform sampling of the unit sphere, and pick samples $\theta$'s that contain discriminant information between $I_\mu$ and $I_\nu$ instead. This idea was for instance used in \citet{deshpande2018generative}, where the authors first calculate a linear discriminant subspace and then measure the empirical SW distance by setting the $\theta$'s to be the discriminant components of the subspace. 

A similarly flavored but less heuristic approach is to use the maximum sliced $p$-Wasserstein (max-SW) distance, which is an alternative OT metric defined as: 
\begin{equation}
     \msw (I_\mu,I_\nu) = \max_{\theta \in \mathbb{S}^{d-1}} W_p\big(\mathcal{R} I_\mu(\cdot,\theta),\mathcal{R} I_\nu(\cdot,\theta) \big)
 \label{eq:msw}
 \end{equation}

Given that $W_p$ is a distance, it is easy to show that max-$SW_p$ is also a distance: we will prove in Section~\ref{subsection:GSWandmaxGSW} that the metric axioms hold for the maximum generalized sliced-Wasserstein distance, which contains the max-SW distance as a special case. 
\begin{figure}[t!]
    \centering
    \includegraphics[width=\columnwidth]{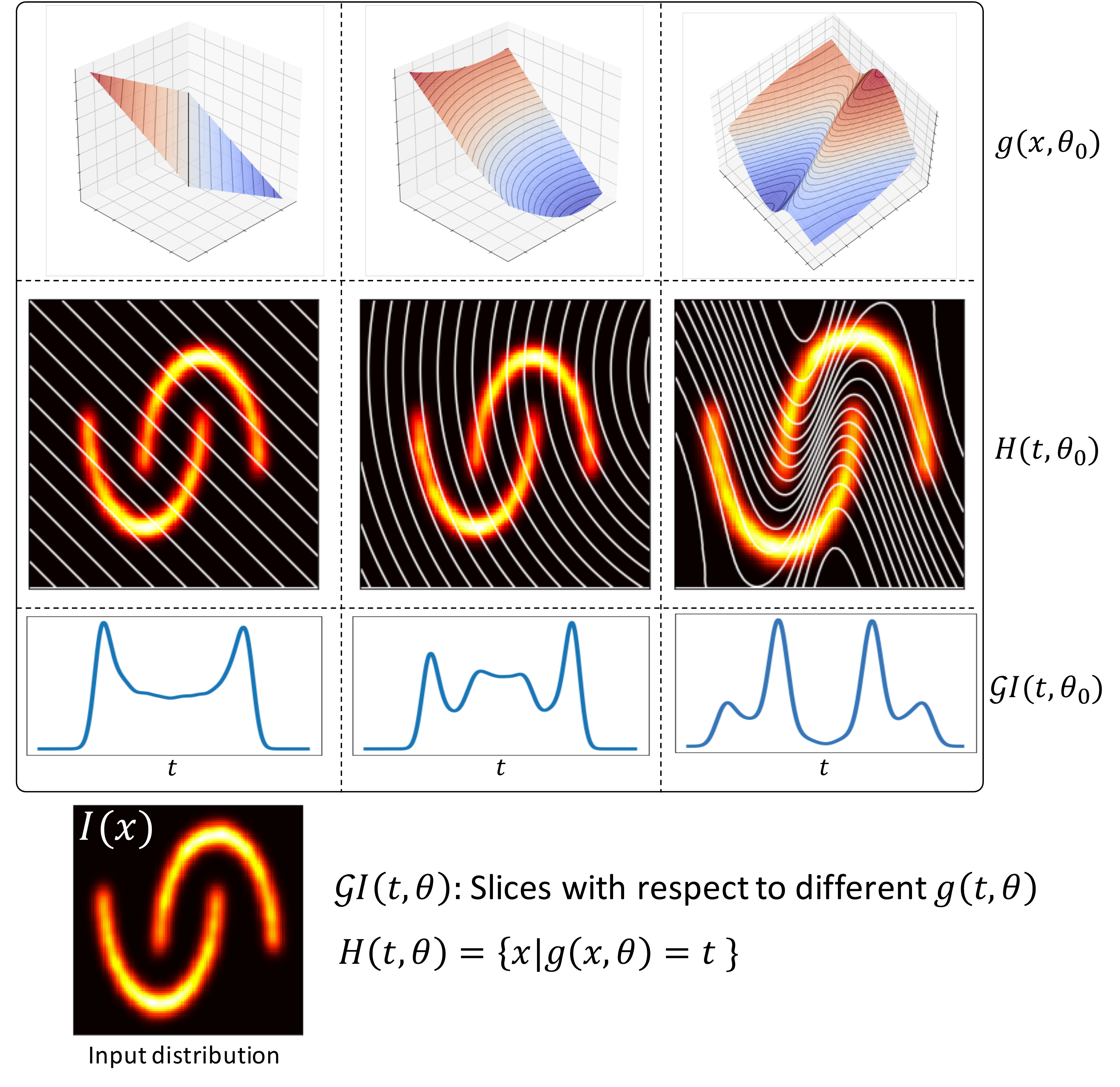}
    \caption{Visualizing the slicing process for classical and generalized Radon transforms for the Half Moons distribution. The slices $\mathcal{G}I(t,\theta)$ follow Equation \eqref{eq:gradon}.}
    \label{fig:fig1}
\end{figure}

\comment{\subsection{Generalized Radon Transform}
\textcolor{blue}{[Kimia] Maybe we should move this subsection in section 3.}}

\section{Generalized Sliced-Wasserstein Distances} 

We propose in this paper to extend the definition of the sliced-Wasserstein distance to formulate a new optimal transport metric, which we call the generalized sliced-Wasserstein (GSW) distance. The GSW distance is obtained using the same procedure as for the SW distance, except that here, the one-dimensional representations are acquired through nonlinear projections. In this section, we first review the generalized Radon transform, which is used to project the high-dimensional distributions, and we then formally define the class of GSW distances. We also extend the concept of max-SW distance to the class of maximum generalized sliced-Wasserstein (max-GSW) distances.

\subsection{Generalized Radon Transform}

The generalized Radon transform (GRT) extends the original idea of the classical Radon transform introduced by \citet{radon1917uber} from integration over hyperplanes of $\mathbb{R}^d$ to integration over hypersurfaces, \textit{i.e.} $(d-1)$-dimensional manifolds \cite{beylkin1984inversion,denisyuk1994inversion,ehrenpreis2003universality,gel1969differential,kuchment2006generalized,homan2017injectivity}. The GRT has various applications, including Thermoacoustic Tomography, where the hypersurfaces are spheres, and Electrical Impedance Tomography, which requires integration over hyperbolic surfaces.

To formally define the GRT, we introduce a function $g$ defined on $\mathcal{X} \times (\mathbb{R}^n \backslash \{ 0 \})$ with $\mathcal{X} \subset \mathbb{R}^d$. We say that $g$ is a \emph{defining function} when it satisfies the four conditions below:


\begin{assumption}
    \label{asmp:g_real}
    $g$ is a real-valued $C^\infty$ function on $\mathcal{X} \times (\mathbb{R}^n \backslash \{ 0 \})$
\end{assumption}
\begin{assumption}
    \label{asmp:g_hom}
    $g(x, \theta)$ is homogeneous of degree one in $\theta$, \textit{i.e.},
    \begin{equation*}
        \forall \lambda \in \mathbb{R},\; g(x, \lambda \theta) = \lambda g(x, \theta)
    \end{equation*}
\end{assumption}
\begin{assumption}
    \label{asmp:g_deg}
    $g$ is non-degenerate in the sense that
    \begin{equation*}
        \forall (x,\theta) \in \mathcal{X} \times \mathbb{R}^n \backslash \{ 0 \},\; \frac{\partial g}{\partial x}(x, \theta) \neq 0 
    \end{equation*}
\end{assumption}
\begin{assumption}
    \label{asmp:g_hes}
    The mixed Hessian of $g$ is strictly positive, \textit{i.e.} 
    \begin{equation*}
        \text{det}\left( \left( \frac{\partial^2 g}{\partial x^i \partial \theta^j} \right)_{i,j} \right) > 0
    \end{equation*}
\end{assumption}

Then, the GRT of $I\in L^1(\mathbb{R}^d)$ is the integration of $I$ over hypersurfaces characterized by the level sets of $g$, which are characterized by $H_{t,\theta} = \{ x \in \mathcal{X}\ |\ g(x, \theta) = t \}$.

Let $g$ be a defining function. The generalized Radon transform of $I$, denoted by $\mathcal{G}I$, is then formally defined as:
\begin{equation}
    \mathcal{G}I(t,\theta)=\int_{\mathbb{R}^d} I(x)\delta(t-g(x,\theta))dx
    \label{eq:gradon}
\end{equation}
Note that the standard Radon transform is a special case of the GRT with $g(x,\theta)= \langle x, \theta \rangle$. Figure~\ref{fig:fig1} illustrates the slicing process for standard and generalized Radon transforms for the Half Moons dataset as input.

\subsection{Generalized Sliced-Wasserstein and Maximum Generalized Sliced-Wasserstein Distances} \label{subsection:GSWandmaxGSW}

Following the definition of the SW distance in Equation~\eqref{eq:SW}, we define the generalized sliced $p$-Wasserstein distance using the generalized Radon transform as:
\begin{equation}
    GSW_p(I_\mu,I_\nu) = \left(\int_{\Omega_\theta} W^p_p\big(\mathcal{G} I_\mu(\cdot,\theta), \mathcal{G} I_\nu(\cdot,\theta)\big) d\theta\right)^{\frac{1}{p}} \label{eq:GSW}
\end{equation}
where $\Omega_\theta$ is a compact set of feasible parameters for $g(\cdot,\theta)$ (e.g., $\Omega_\theta=\mathbb{S}^{d-1}$ for $g(\cdot,\theta) = \langle \cdot, \theta \rangle$). 

The GSW distance can also suffer from the projection complexity issue described in Section~\ref{subsection:SWandmaxSW}; that is why we formulate the maximum generalized sliced $p$-Wasserstein distance, which generalizes the max-SW distance as defined in \eqref{eq:msw}:
\begin{eqnarray}
\mgsw(I_\mu,I_\nu) = \max_{\theta\in\Omega_\theta} W_p\big(\mathcal{G} I_\mu(\cdot, \theta),\mathcal{G} I_\nu(\cdot, \theta)\big) 
 \label{eq:MaxGSW}
\end{eqnarray}


\begin{proposition}
    The generalized sliced $p$-Wasserstein distance and the maximum generalized sliced $p$-Wasserstein distance are, indeed, distances over $\mathcal{P}_p(\Omega)$ if and only if the generalized Radon transform is injective.
\end{proposition}
\begin{proof}
    The non-negativity and symmetry are direct consequences of the fact that the Wasserstein distance is a metric \cite{villani2008optimal}: see supplementary material. 

    We prove the triangle inequality for $GSW_p$ and max-$GSW_p$. Let $\mu_1$, $\mu_2$ and $\mu_3$ in $\mathcal{P}_p(\Omega)$. Since the Wasserstein distance satisfies the triangle inequality, we have, for all $\theta \in \Omega_\theta$,
        \begin{align*}
            W_p(\mathcal{GI}_{\mu_1}(\cdot, \theta), \mathcal{GI}_{\mu_3}(\cdot, \theta)) &\leq W_p(\mathcal{GI}_{\mu_1}(\cdot, \theta), \mathcal{GI}_{\mu_2}(\cdot, \theta)) \\
            &\;\;\; + W_p(\mathcal{GI}_{\mu_2}(\cdot, \theta), \mathcal{GI}_{\mu_3}(\cdot, \theta))
        \end{align*}
    Therefore, we can write:
    \begin{align}
        GSW_p&(I_{\mu_1}, I_{\mu_3}) = \left( \int_{\Omega_\theta} W_p^p(\mathcal{G}I_{\mu_1}(\cdot, \theta), \mathcal{G}I_{\mu_3}(\cdot, \theta)) d\theta \right)^{\frac{1}{p}} \nonumber \\
        &\leq \left( \int_{\Omega_\theta} \big( W_p(\mathcal{G}I_{\mu_1}(\cdot, \theta), \mathcal{G}I_{\mu_2}(\cdot, \theta)) \right. \nonumber \\
        &\;\;\; + \left. W_p(\mathcal{G}I_{\mu_2}(\cdot, \theta), \mathcal{G}I_{\mu_3}(\cdot, \theta)) \big)^p d\theta \vphantom{\int} \right)^{\frac{1}{p}} \nonumber \\
        &\leq \left( \int_{\Omega_\theta} W_p^p(\mathcal{G}I_{\mu_1}(\cdot, \theta), \mathcal{G}I_{\mu_2}(\cdot, \theta)) d\theta \right)^{\frac{1}{p}} \nonumber \\
        &\;\;\; + \left( \int_{\Omega_\theta} W_p^p(\mathcal{G}I_{\mu_2}(\cdot, \theta), \mathcal{G}I_{\mu_3}(\cdot, \theta)) d\theta \right)^{\frac{1}{p}} \label{eq:gsw_minkowski}
    \end{align}
    where inequality~\eqref{eq:gsw_minkowski} follows from the application of the Minkowski inequality in $L^p(\Omega_\theta)$. We conclude that $GSW_p$ satisfies the triangle inequality.
    
    Let $\theta^*=\argmax_{\theta \in \Omega_\theta} W_p(\mathcal{GI}_{\mu_1}(\cdot, \theta), \mathcal{GI}_{\mu_3}(\cdot, \theta))$; then,
    \begin{align*}
       &\text{max-}GSW_p(I_{\mu_1}, I_{\mu_3}) \\
        &= \max_{\theta \in \Omega_\theta}\ W_p(\mathcal{GI}_{\mu_1}(\cdot, \theta), \mathcal{GI}_{\mu_3}(\cdot, \theta)) \\
        &= W_p(\mathcal{GI}_{\mu_1}(\cdot, \theta^*), \mathcal{GI}_{\mu_3}(\cdot, \theta^*)) \\
        &\leq W_p(\mathcal{GI}_{\mu_1}(\cdot, \theta^*), \mathcal{GI}_{\mu_2}(\cdot, \theta^*)) \\
        &\;\;\; + W_p(\mathcal{GI}_{\mu_2}(\cdot, \theta^*), \mathcal{GI}_{\mu_3}(\cdot, \theta^*))\\
        &\leq \max_{\theta \in \Omega_\theta} W_p(\mathcal{GI}_{\mu_1}(\cdot, \theta), \mathcal{GI}_{\mu_2}(\cdot, \theta)) \\
        &\;\;\; + \max_{\theta \in \Omega_\theta} W_p(\mathcal{G}I_{\mu_2}(\cdot, \theta), \mathcal{G}I_{\mu_3}(\cdot, \theta)) \\
    &\leq \text{max-}GSW_p(I_{\mu_1}, I_{\mu_2}) +  \text{max-}GSW_p(I_{\mu_2}, I_{\mu_3})
    \end{align*}
    So max-$GSW_p$ also satisfies the triangle inequality.
    
    Since $W_p(\mu, \mu) = 0$ for any $\mu$, we have $GSW_p(I_\mu, I_\nu) = 0$ and max-$GSW_p(I_\mu, I_\nu) = 0$. Now, $GSW_p(I_\mu, I_\nu) = 0$ or max-$GSW_p(I_\mu, I_\nu) = 0$ is equivalent to $\mathcal{G}I_\mu(\cdot, \theta) = \mathcal{G}I_\nu(\cdot, \theta)$ for almost all $\theta \in \Omega_\theta$. Therefore, GSW and max-GSW are distances if and only if $\mathcal{G}I_\mu(\cdot, \theta) = \mathcal{G}I_\nu(\cdot, \theta)$ implies $\mu = \nu$, \textit{i.e.} the GRT is injective. \end{proof}
    

\begin{remark}
    If the chosen generalized Radon transform is not injective, then we can only say that the GSW and max-GSW distances are pseudo-metrics: they still satisfy non-negativity, symmetry, the triangle inequality, and $GSW_p(I_\mu, I_\mu) = 0$ and $\text{max-}GSW_p(I_\mu, I_\mu) = 0$. 
\end{remark}


\subsection{Injectivity of the Generalized Radon Transform}

We have shown that the injectivity of the GRT is crucial for the GSW and max-GSW distances to be, indeed, distances between probability measures. Here, we enumerate some of the known defining functions that lead to injective GRTs.  

The investigation of the sufficient and necessary conditions for showing the injectivity of GRTs is a long-standing topic \cite{beylkin1984inversion,homan2017injectivity,uhlmann2003inside,ehrenpreis2003universality}.
The circular defining function, $g(x,\theta) = \|x-r*\theta\|_2$ with $r\in\mathbb{R}^+$ and $\Omega_\theta = \mathbb{S}^{d-1}$ was shown to provide an injective GRT \cite{kuchment2006generalized}. More interestingly, homogeneous polynomials with an odd degree also yield an injective GRT \cite{rouviere}, \textit{i.e.}
\begin{align*}
    g(x,\theta) = \sum_{|\alpha| = m} \theta_\alpha x^\alpha,
\end{align*}
where we use the multi-index notation $\alpha = (\alpha_1, \dots, \alpha_{d_\alpha}) \in \mathbb{N}^{d_\alpha}$, $|\alpha| = \sum_{i=1}^{d_\alpha} \alpha_i$, and $x^\alpha = \prod_{i=1}^{d_\alpha} x_i^{\alpha_i}$. Here, the summation iterates over all possible multi-indices $\alpha$, such that $|\alpha| = m$, where $m$ denotes the degree of the polynomial and $\theta_\alpha \in \mathbb{R}$. The parameter set for homogeneous polynomials is then set to $\Omega_\theta=\mathbb{S}^{d_\alpha-1}$.  We can observe that choosing $m=1$ reduces to the linear case $\langle x,\theta\rangle$, since the set of the multi-indices with $|\alpha|=1$ becomes $\{ (\alpha_1, \dots, \alpha_d); \alpha_i = 1 \text{ for a single } i\in \llbracket 1, d \rrbracket, \text{ and } \alpha_j = 0, \quad \forall j \neq i\}$ and contains $d$ elements.

\section{Numerical Implementation}

In this section, we briefly review the numerical methods used to compute the GSW and max-GSW distances. 
\begin{algorithm}[t!]
\caption{\ GSW Distance}
\label{alg:algo1}
\begin{algorithmic}
\INPUT{$\{x_i\sim I_\mu\}_{i=1}^N$, ~~$\{y_i\sim I_\nu\}_{i=1}^N$,\; order $p$, \\ 
\hspace{.2in}  number of slices $L$, defining function $g$}
\STATE Initialize $d=0$
\FOR{$l=1$ to $L$}
\STATE Sample $\theta_l$ from $\Omega_\theta$ uniformly
\STATE Compute $\hat{x}_i=g(x_i,\theta_l)$ and $\hat{y}_i=g(y_i,\theta_l)$ for each $i$
\STATE Sort $\hat{x}_i$ and $\hat{y}_j$ in ascending order s.t. $\hat{x}_{i[n]}\leq\hat{x}_{i[n+1]}$ and $\hat{y}_{j[n]}\leq\hat{y}_{j[n+1]}$ 
\STATE $d=d+\frac{1}{L}\sum_{n=1}^N |\hat{x}_{i[n]}-\hat{y}_{i[n]}|^p$
\ENDFOR
\OUTPUT{$d^{\frac{1}{p}}\approx GSW_p(I_\mu,I_\nu)$}
\end{algorithmic}
\end{algorithm}

\begin{algorithm}[t!]
\caption{\ Max-GSW Distance}
\label{alg:algo2}
\begin{algorithmic}
\INPUT{$\{x_i\sim I_\mu\}_{i=1}^N$, ~~$\{y_j\sim I_\nu\}_{j=1}^N$, \\ 
\hspace{.2in} order $p$, defining function $g(x,\theta)$}
\STATE Randomly initialize $\theta\in\Omega_\theta$
\WHILE{$\theta$ has not converged}
\STATE Compute $\hat{x}_i=g(x_i,\theta_l)$ and $\hat{y}_i=g(y_i,\theta_l)$ for each $i$
\STATE Sort $\hat{x}_i$ and $\hat{y}_j$ in ascending order s.t. $\hat{x}_{i[n]}\leq\hat{x}_{i[n+1]}$ and $\hat{y}_{j[n]}\leq\hat{y}_{j[n+1]}$ 
\STATE $\theta=\underset{\Omega_\theta}{Proj}(ADAM (\nabla_\theta(\frac{1}{N}\sum_{n=1}^N |\hat{x}_{i[n]}-\hat{y}_{j[n]}|^p),\theta ))$
\ENDWHILE
\STATE Sort $\hat{x}_i$ and $\hat{y}_i$ in ascending order 
\STATE $d=\frac{1}{N}\sum_{n=1}^N |\hat{x}_{i[n]}-\hat{y}_{i[n]}|^p$

\OUTPUT{$d^{\frac{1}{p}} \approx \mgsw(I_\mu,I_\nu)$}
\end{algorithmic}
\end{algorithm}

\subsection{Generalized Radon Transforms of Empirical PDFs}

In most machine learning applications, we do not have access to the distribution $I_\mu$ but to a set of samples $\{ x_i \}_{i=1}^N$ drawn from $I_\mu$, for which the empirical density is:
\begin{eqnarray*}
I_\mu(x)\approx \frac{1}{N}\sum_{i=1}^N \delta(x-x_i)
\end{eqnarray*} 
 The GRT of the empirical density is then given by:
\begin{eqnarray*}
    \mathcal{G}I_\mu(t,\theta)\approx \frac{1}{N}\sum_{i=1}^N \delta \big(t-g(x_i,\theta)\big)
\end{eqnarray*} 
 Moreover, for high-dimensional problems, estimating $I_\mu$ in $\mathbb{R}^d$ requires a large number of samples. However,  the projections of $I_\mu$, $\mathcal{G}I(\cdot,\theta)$, are one-dimensional and it may not be critical to have a large number of samples to estimate these one-dimensional densities. 

\subsection{Numerical Implementation of GSW Distances} 

Let $\{x_i\}_{i=1}^N$ and $\{y_j\}_{j=1}^N$ be samples respectively drawn from $I_\mu$ and $I_\nu$, and let $g(\cdot,\theta)$ be a defining function.  Following the work of \citet{kolouri2018sliced}, the Wasserstein distance between one-dimensional distributions $\mathcal{G}I_\mu(\cdot,\theta)$ and $\mathcal{G}I_\nu(\cdot,\theta)$ can be calculated from sorting their samples and calculating the $L_p$ distance between the sorted samples. In other words, the GSW distance between $I_\mu$ and $I_\nu$ can be approximated from their samples as follows:
\begin{equation*}
    GSW_p(I_\mu,I_\nu)\approx \Big(\frac{1}{L} \sum_{l=1}^L \sum_{n=1}^N |g(x_{i[n]},\theta_l)-g(y_{j[n]},\theta_l)|^p \Big)^{\frac{1}{p}}
\end{equation*}
where $i[m]$ and $j[n]$ are the indices of sorted $\{g(x_i,\theta)\}_{i=1}^N$ and $\{g(y_j,\theta)\}_{j=1}^N$. The procedure to approximate the GSW distance is summarized in Algorithm \ref{alg:algo1}.

\subsection{Numerical Implementation of max-GSW Distances} 

To compute the max-GSW distance, we perform an EM-like optimization scheme: (a) for a fixed $\theta$, $g(x_i,\theta)$ and $g(y_i,\theta)$ are sorted to compute $W_p$, (b) $\theta$ is updated with:
\begin{equation*}
    \theta = \underset{\Omega_\theta}{Proj} \big(ADAM\big(\nabla_\theta(\frac{1}{N}\sum_{n=1}^N |g(x_{i[n]},\theta)-g(y_{j[n]},\theta)|^p),\theta\big)\big)
\end{equation*}
where $ADAM$ refers to the ADAM optimizer \cite{kingma2014adam} and $\underset{\Omega_\theta}{Proj}(\cdot)$ is the operator projecting $\theta$ onto $\Omega_\theta$. For instance, when $\theta\in\mathbb{S}^{n-1}$, $\underset{\Omega_\theta}{Proj}(\theta)=\frac{\theta}{\|\theta\|}$. 

\begin{remark} Here, we find the optimal $\theta$ by optimizing the actual $W_p$, as opposed to the heuristic approaches proposed in \citet{deshpande2018generative} and \citet{kolouri2018sliced}, where the pseudo-optimal slice is found via perceptrons or penalized linear discriminant analysis  \cite{wang2011penalized}.
\end{remark}

Finally, once convergence is reached, the max-GSW distance is approximated with: 
\begin{equation*}
    \mgsw(I_\mu,I_\nu)\approx \big(\frac{1}{N}\sum_{n=1}^N |g(x_{i[n]},\theta^*)-g(y_{j[n]},\theta^*)|^p \big)^{\frac{1}{p}}
\end{equation*}
The whole procedure is summarized in Algorithm \ref{alg:algo2}. 

\section{Experiments}

\subsection{Generalized Sliced-Wasserstein Flows}

Our first experiment demonstrates the effects of the choice of the GSW distance in its purest form by considering the following problem: \;
$\operatorname{min}_{\mu} GSW_p(\mu,\nu)$, where $\nu$ is a target distribution and $\mu$ is the source distribution, which is initialized to be the normal distribution. The optimization is then solved iteratively via
\begin{equation*}
\partial_t\mu_t= -\nabla GSW_p(\mu_t,\nu), ~~\mu_0=\mathcal{N}(0,1)
\end{equation*}
We used 5 well-known distributions as the target, namely the 25-Gaussians, 8-Gaussians, Swiss Roll, Half Moons and Circle distributions. We compare linear (\textit{i.e.}, SW distance), circular, homogeneous polynomial of degree 3 and homogeneous polynomial of degree 5 as defining functions. We used the exact same optimization scheme for all methods, with $L=10$ random projections, and measured the 2-Wasserstein distance between $\mu_t$ and $\nu$ at each iteration of the optimization (via solving a linear programming at each step). We repeated each experiment 100 times and reported the mean and standard deviation of the 2-Wasserstein distance for all five target datasets in Figure \ref{fig:exp1}. While the choice of the defining function $g(\cdot,\theta)$ is data-dependent, one can see that the homogeneous polynomial of degree 3 is among the top two performers for all datasets. 

\begin{figure}[t!]
    \centering
    \includegraphics[width=\columnwidth]{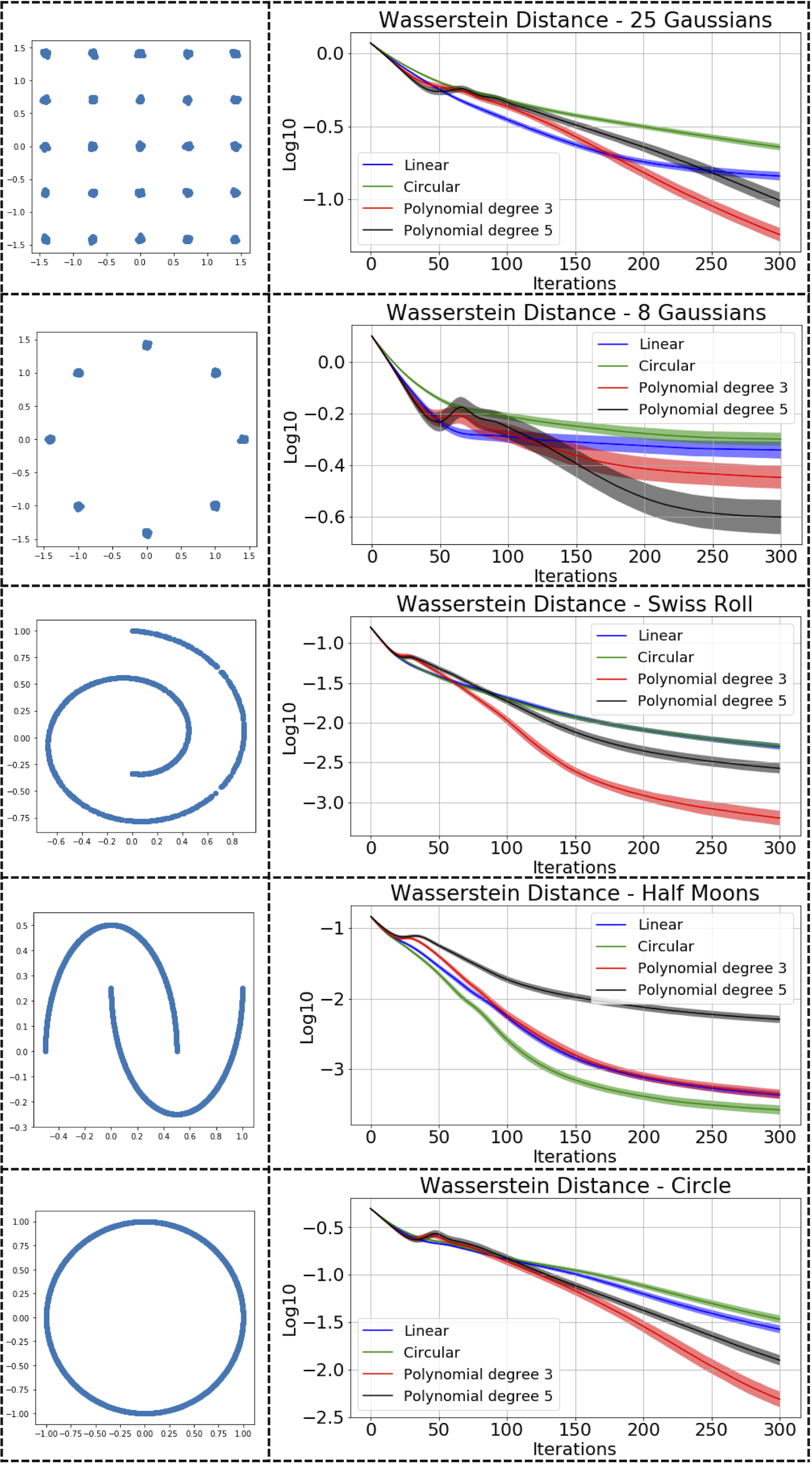}
    \caption{Log 2-Wasserstein distance between the source and target distributions as a function of the number of iterations for 5 classical target distributions.}
    \label{fig:exp1}
    \vspace{-.1in}
\end{figure}

For clarity purposes, we chose to not report the \mgsw~ results for the same experiment in Figure \ref{fig:exp1}. These results are included in the supplementary material. 

\subsection{Generative Modeling via Auto-Encoders}

We now demonstrate the application of the GSW and max-GSW distances in generative modeling. We specifically use the recently proposed Sliced-Wasserstein Auto-Encoder (SWAE) \cite{kolouri2018sliced} framework, which penalizes the distribution of the encoded data in the latent space of the auto-encoder to follow a prior samplable distribution, $p_Z$. More precisely, let $\{x_n\sim p_X\}_{n=1}^N$ be i.i.d. samples from $p_X$, $\phi(x,\gamma_\phi):\mathcal{X}\rightarrow \mathcal{Z}$ and $\psi(z,\gamma_\psi):\mathcal{Z}\rightarrow \mathcal{X}$ be the parametric encoder and decoder (e.g., CNNs) with parameters $\gamma_\phi$ and $\gamma_\psi$, respectively. Then SWAE's objective function \cite{kolouri2018sliced} is defined as:
\begin{equation}
    \min_{\gamma_\phi,\gamma_\psi} \mathbb{E}_x [ c(x,\psi(\phi(x,\gamma_\phi),\gamma_\psi))]+\lambda SW(p_{\phi(x,\gamma_\phi)},p_Z)
\label{eq:swae}
\end{equation}
where $\lambda$ is the regularizer coefficient for matching the encoded distribution to $p_Z$.  Here, we substitute the SW distance in Equation~\eqref{eq:swae} with GSW and max-GSW distances. Specifically, we encode the MNIST dataset \cite{lecun1998gradient} into the encoder's latent space and enforce the distribution of the embedded data to follow a specific prior distribution, e.g. the Swiss Roll distribution as shown in Figure \ref{fig:schematic}, while we simultaneously enforce the encoded features to be decodable to the original input images.
\begin{figure}[t]
    \centering
    \includegraphics[width=\columnwidth]{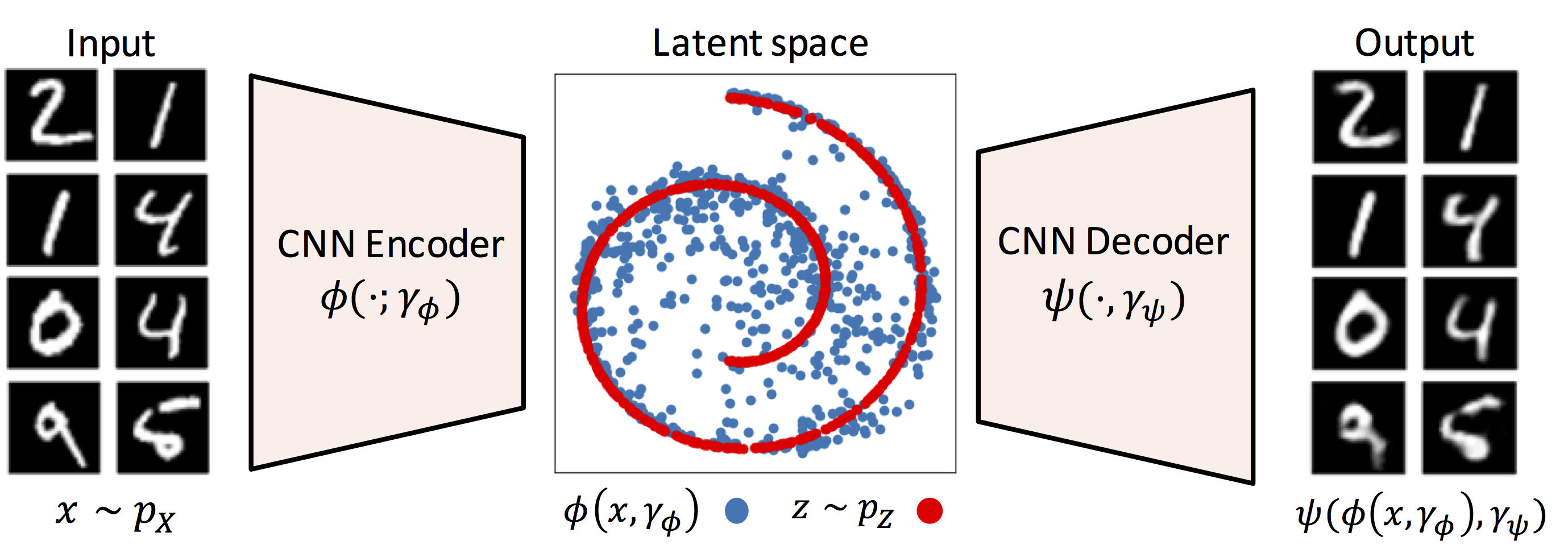}
    \caption{The SWAE architecture. The embedded data in the latent space is enforced to follow a prior samplable distribution $p_Z$.}
    \label{fig:schematic}
    \vspace{-.15in}
\end{figure}

We ran the optimization in Equation~\eqref{eq:swae} with GSW distances, which we denote as GSWAE, with linear, circular, and homogeneous polynomial of degree 3. At each iteration, we measured the 2-Wasserstein distance between the embedded distribution and the prior distribution, $W_2(p_{\phi(x,\gamma_\phi)},p_Z)$, and also between the input distribution and the distribution of the reconstructed samples, $W_2(p_{\psi(\phi(x,\gamma_\phi),\gamma_\psi)},p_X)$. Each experiment was repeated $50$ times and the average 2-Wasserstein distances are reported in Figure \ref{fig:exp2}. The middle row in Figure \ref{fig:exp2} shows samples from $p_Z$ and $\phi(x,\gamma_\phi)$ for $x\sim p_X$, and the last row shows decoded random samples,  $\psi(z,\gamma_\psi)$ for $z\sim p_Z$. Similar to the previous experiment, we see that the GSWAE with a polynomial defining function, captures the nonlinear geometry of the input samples better.

We also compare the performance of GSWAE and Max-GSWAE with those of SWAE and WAE-GAN \cite{tolstikhin2018wasserstein}. In particular, we use the improved Wasserstein-GAN \cite{gulrajani2017improved}, which is among the state-of-the-art adversarial training methods, in the embedding space of the Wasserstein auto-encoder \cite{tolstikhin2018wasserstein}. The adversary was chosen to be a multi-layer perceptron. Similar to the previous experiments, we measured the 2-Wasserstein distance between the input and output distributions as well as the latent and prior distributions. Each experiment was repeated $10$ times, and the average 2-Wasserstein distances are reported in Figure \ref{fig:exp3}. It can be seen that, while WAE-GAN provides a better matching of distributions in the latent space, the results of max-GSWAE distances are on par with the WAE-GAN. In addition, by comparing the distance between input and output distributions of the auto-encoder, it seems that max-GSWAE provides a better objective function to train such networks. 
\begin{figure}[t]
    \centering
    \includegraphics[width=\columnwidth]{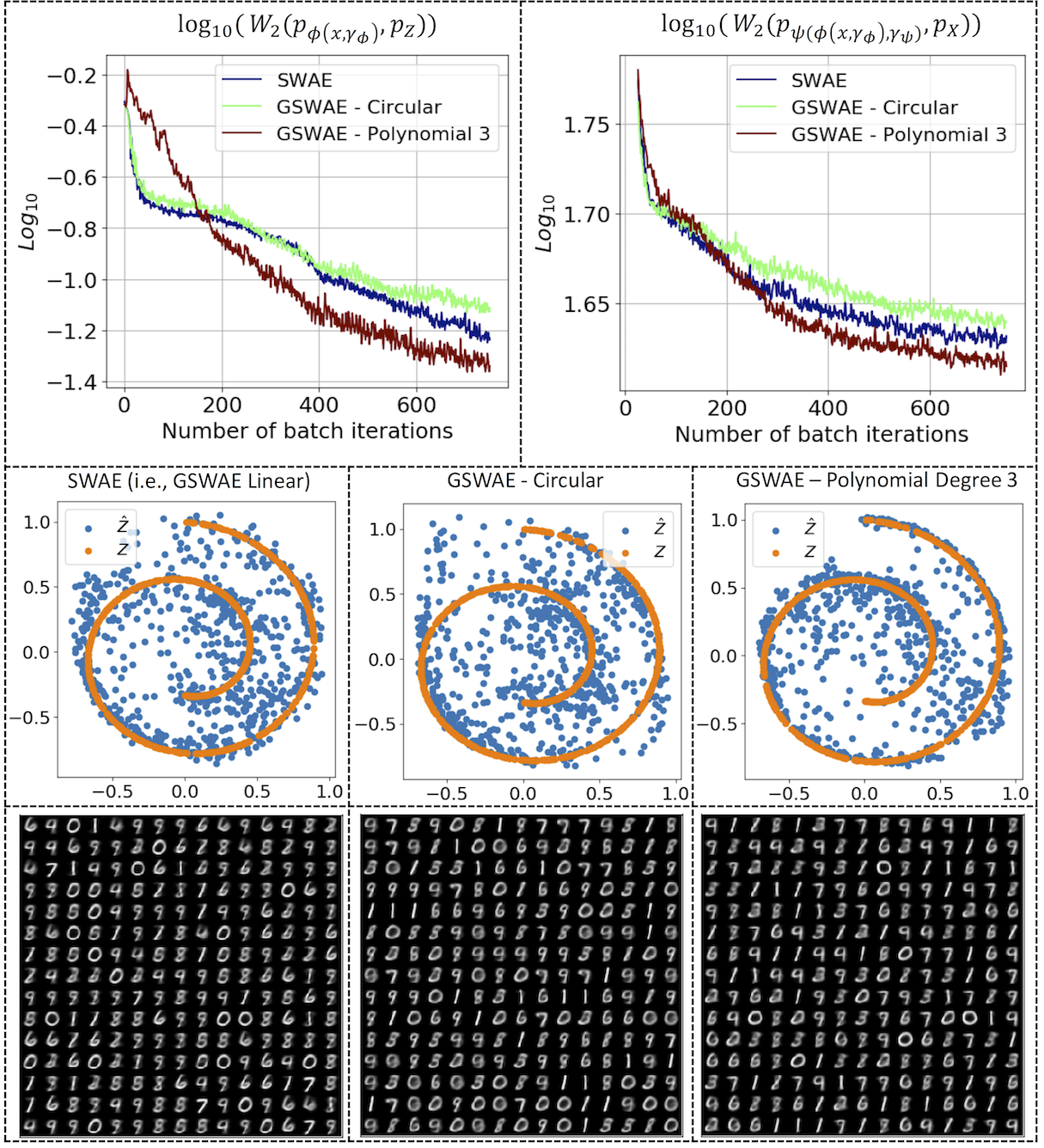}
    \vspace{-.2in}
    \caption{$2$-Wasserstein distance between $p_Z$ and $p_{\phi(x,\gamma_\phi)}$ and between $p_X$ and $p_{\psi(\phi(x,\gamma_\phi),\gamma_\psi)}$ at different batch iterations for SWAE and GSWAE with circular and polynomial of degree 3 defining functions. }
    \label{fig:exp2}
     \vspace{-.2in}
\end{figure}
\begin{figure}[t]
    \centering
    \includegraphics[width=\columnwidth]{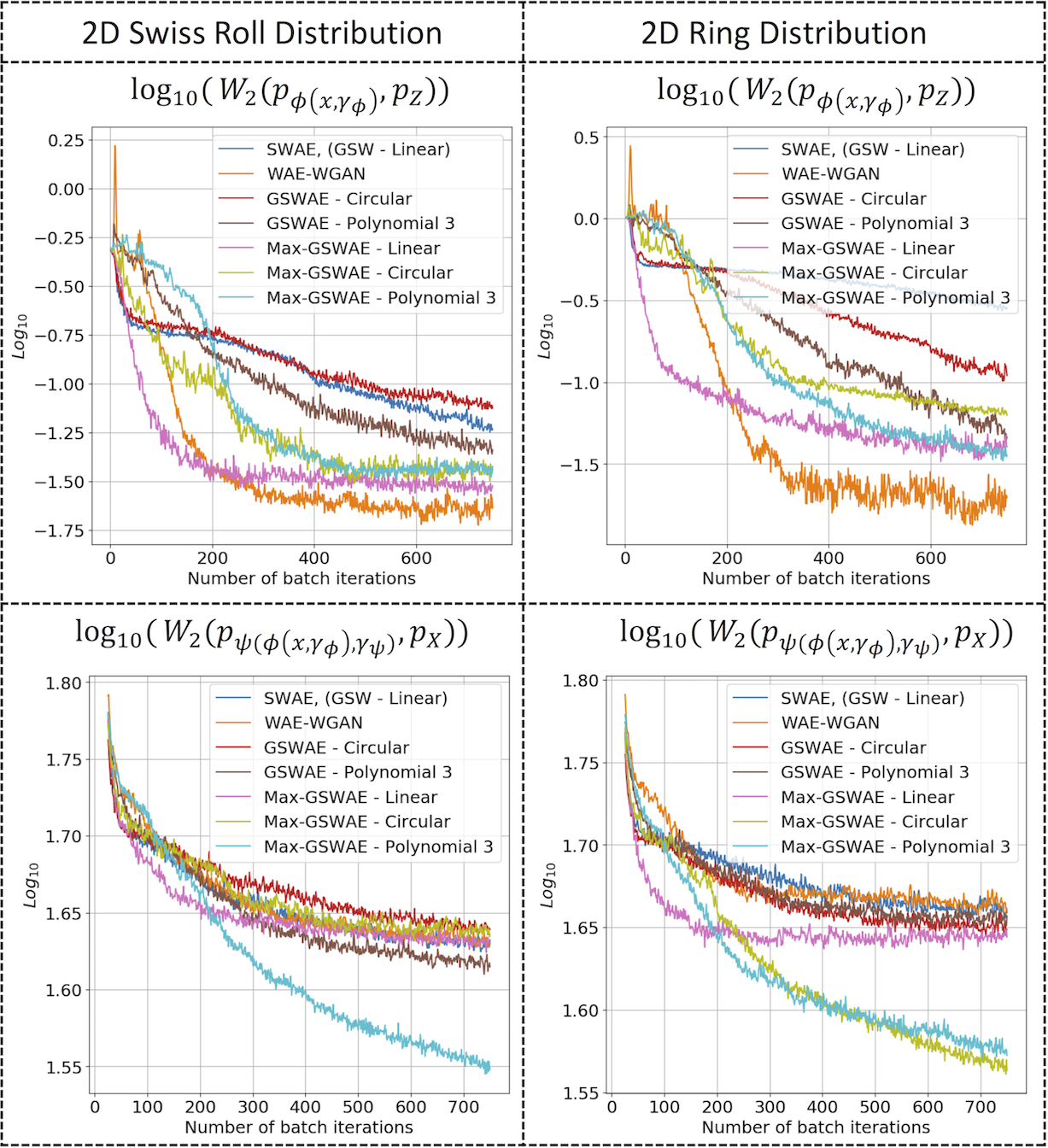}
    \vspace{-.2in}
    \caption{The $2$-Wasserstein distance between $p_Z$ and $p_{\phi(x,\gamma_\phi)}$ and between $p_X$ and $p_{\psi(\phi(x,\gamma_\phi),\gamma_\psi)}$ at different batch iterations for SWAE and WAE-GAN compared to GSWAE and Max-GSWAE with circular and polynomial of degree 3 defining functions. }
    \label{fig:exp3}
    \vspace{-.2in}
\end{figure}

\section{Conclusion}

We introduced a new family of optimal transport metrics for probability measures that generalizes the sliced-Wasserstein distance: while the latter is based on linear slicing of distributions, we propose to perform nonlinear slicing. We provided theoretical conditions that yield the generalized sliced-Wasserstein distance to be, indeed, a distance function, and we empirically demonstrated the superior performance of the GSW and max-GSW distances over the classical sliced-Wasserstein distance in various generative modeling applications. As future work, we plan to study the existing connection between adversarial training and max-GSW distances by showing the defining function for GRTs can be approximated with neural networks. 

\section{Acknowledgement}

This work was partially supported by the United States Air Force and DARPA under Contract No. FA8750-18-C-0103. Any opinions, findings and conclusions or recommendations expressed in this material are those of the author(s) and do not necessarily reflect the views of the United States Air Force and DARPA.

\clearpage
\bibliography{gsw}
\bibliographystyle{icml2019}

\newpage
\section{Supplementary material}

\section{Non-negativity and Symmetry of the GSW and max-GSW Distances}

We prove that the GSW and max-GSW distances satisfy non-negativity and symmetry, using the fact that the $p$-Wasserstein distance is known to be a proper distance function \cite{villani2008optimal}. Let $\mu$ and $\nu$ be in $\mathcal{P}_p(\Omega)$. 

\subsection{Non-negativity}

We use the non-negativity of the $p$-Wasserstein distance, $\textit{i.e.}$ $W_p(\mu, \nu) \geq 0$ for any $\mu$, $\nu$ in $\mathcal{P}_p(\Omega)$, to show that the GSW and max-GSW distances are non-negative as well:
\begin{align*}
    GSW_p(I_\mu,I_\nu) &= \left(\int_{\Omega_\theta} W^p_p\big(\mathcal{G} I_\mu(.,\theta),\mathcal{G} I_\nu(.,\theta)\big) d\theta\right)^{\frac{1}{p}}\\
    &\geq \left(\int_{\Omega_\theta} (0)^p d\theta\right)^{\frac{1}{p}}=0 \\
    & \\
    \text{max-}GSW_p(I_\mu, I_\nu) &= \max_{\theta \in \Omega_\theta} W_p\big(\mathcal{G}I_\mu(\cdot, \theta), \mathcal{G}I_\nu(\cdot, \theta)\big) \\
    &= W_p\big(\mathcal{G}I_\mu(\cdot, \theta^*), \mathcal{G}I_\nu(\cdot, \theta^*)\big) \\
    &\geq 0
\end{align*}
where $\theta^* = \argmax_{\theta \in \Omega_\theta} W_p(\mathcal{G}I_\mu(\cdot, \theta), \mathcal{G}I_\nu(\cdot, \theta))$.

\subsection{Symmetry}

Since the $p$-Wasserstein distance is symmetric, we have $W_p(\mu, \nu) = W_p(\nu, \mu)$. In particular, we can write for all $\theta \in \Omega_\theta$:
\begin{equation} \label{eq:all_symmetry}
    \; W_p(\mathcal{G}I_\mu(\cdot, \theta), \mathcal{G}I_\nu(\cdot, \theta)) = W_p(\mathcal{G}I_\nu(\cdot, \theta), \mathcal{G}I_\mu(\cdot, \theta))
\end{equation}
and,
\begin{equation} \label{eq:max_symmetry}
    \max_{\theta \in \Omega_\theta} W_p(\mathcal{G}I_\mu(\cdot, \theta), \mathcal{G}I_\nu(\cdot, \theta)) = \max_{\theta \in \Omega_\theta} W_p(\mathcal{G}I_\nu(\cdot, \theta), \mathcal{G}I_\mu(\cdot, \theta))
\end{equation}

The symmetry of the GSW and max-GSW distances follows from Equations~\eqref{eq:all_symmetry} and \eqref{eq:max_symmetry} respectively. 




\begin{figure}[t]
    \centering
    \includegraphics[width=\columnwidth]{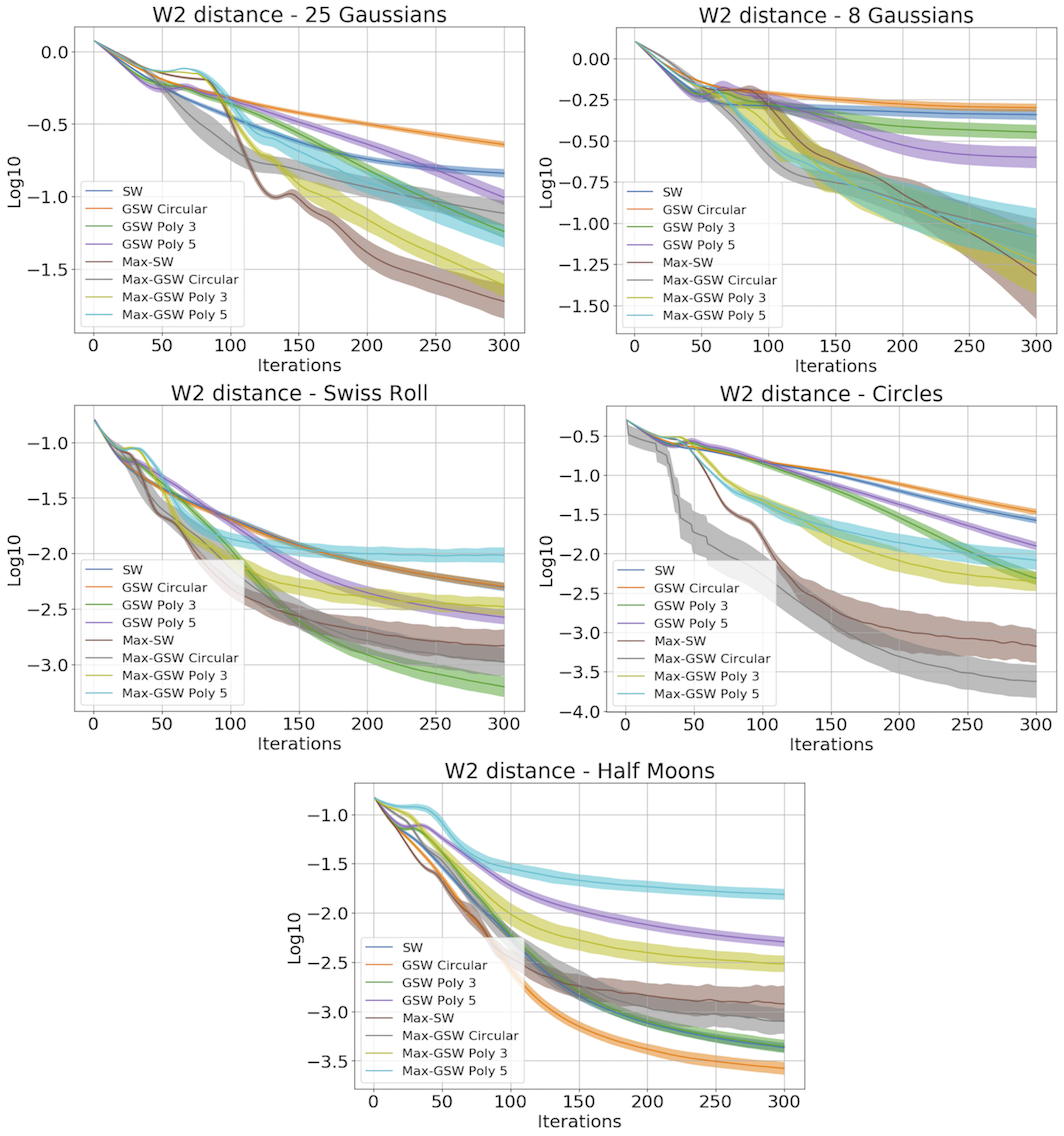}
    \caption{Log 2-Wasserstein distance between the source and target distributions as a function of the number of iterations for 5 classical target distributions using GSW and max-GSW distances.}
    \label{fig:supp_1}
\end{figure}

\section{Additional Experimental Results}


We include the results of maximum generalized sliced-Wasserstein flows on the five datasets used in the main paper, to accompany Figure 4 of our main paper: see Figure \ref{fig:supp_1}. It can be seen that the max-GSW distances, in the majority of cases, improve the performance of GSW. Here it should be noted that GSW distances are calculated based on 10 random projections, while max-GSW distances use only one projection by definition.

\section{Implementation Details} 

The PyTorch \cite{paszke2017automatic} implementation of our paper will be available here\footnote{https://github.com/.../GSW/}. Here we clarify some of the implementation details used in our paper. First, the `critic iteration' for the adversarial training, and the projection maximization for the max-GSW distances, were set to be equal to $50$. For all optimizations, we used ADAM \cite{kingma2014adam} optimizer with learning rate $lr=0.001$ and PyTorch's default momentum parameters. 

We used $3\times 3$ convolutional filters in both encoder and decoder architectures.
Encoder architecture:
\begin{eqnarray*}
x\in \mathbb{R}^{28\times 28} &\rightarrow& Conv_{16} \rightarrow LeakyReLU_{0.2}\\
&\rightarrow& Conv_{16} \rightarrow LeakyReLU_{0.2}\\
&\rightarrow& AvgPool_2\\
&\rightarrow& Conv_{32} \rightarrow LeakyReLU_{0.2}\\
&\rightarrow& Conv_{32} \rightarrow LeakyReLU_{0.2}\\
&\rightarrow& AvgPool_2\\
&\rightarrow& Conv_{64} \rightarrow LeakyReLU_{0.2}\\
&\rightarrow& Conv_{64} \rightarrow LeakyReLU_{0.2}\\
&\rightarrow& AvgPool_2 \rightarrow Flatten\\
&\rightarrow& FC_{128} \rightarrow LeakyReLU_{0.2}\\ 
&\rightarrow& FC_{2}
\end{eqnarray*}

Decoder architecture:
\begin{eqnarray*}
z\in \mathbb{R}^{2} &\rightarrow& FC_{128} \rightarrow LeakyReLU_{0.2}\\
&\rightarrow& FC_{1024} \rightarrow LeakyReLU_{0.2}\\
&\rightarrow& Reshape(4\times 4\times 64) \rightarrow Upsample_2\\
&\rightarrow& Conv_{64} \rightarrow LeakyReLU_{0.2}\\
&\rightarrow& Conv_{64} \rightarrow LeakyReLU_{0.2}\\
&\rightarrow& Upsample_2\\
&\rightarrow& Conv_{32} \rightarrow LeakyReLU_{0.2}\\
&\rightarrow& Conv_{32} \rightarrow LeakyReLU_{0.2}\\
&\rightarrow& Upsample_2\\
&\rightarrow& Conv_{16} \rightarrow LeakyReLU_{0.2}\\
&\rightarrow& Conv_{1}
\end{eqnarray*}

The WGAN in WAE-GAN uses an adversary network. Adversary's architecture: 

\begin{eqnarray*}
z\in \mathbb{R}^{2} &\rightarrow& FC_{500} \rightarrow ReLU\\
&\rightarrow& FC_{500} \rightarrow ReLU\\
&\rightarrow& FC_{500} \rightarrow ReLU\\
&\rightarrow& FC_{1} \rightarrow ReLU
\end{eqnarray*}

\end{document}